\documentclass[runningheads]{llncs}
\usepackage[T1]{fontenc}
\usepackage{amsmath}
\usepackage{amsfonts}
\usepackage{wrapfig}
\usepackage{cite}
\usepackage{graphicx}
\usepackage{booktabs}
\usepackage[colorlinks=true,
    linkcolor=blue,
    citecolor=blue,
    urlcolor=blue]{hyperref}
\usepackage{color}

\begin{document}
\title{Differentiable Zero-One Loss via Hypersimplex Projections}
%
%
\author{
Camilo Gomez\inst{1}\thanks{Corresponding author}
\and
Pengyang Wang\inst{2} 
\and 
Liansheng Tang\inst{1}
}

\authorrunning{C. Gomez et al.}


\institute{
School of Data, Mathematical, and Statistical Sciences, University of Central Florida, Orlando, USA \\
\email{camilo.gomez@ucf.edu, liansheng.tang@ucf.edu}
\and
Department of CIS, University of Macau, Macao, China \\
\email{pywang@um.edu.mo}
}

\maketitle              
\begin{abstract}
Recent advances in machine learning have emphasized the integration of structured optimization components into end-to-end differentiable models, enabling richer inductive biases and tighter alignment with task-specific objectives. In this work, we introduce a novel differentiable approximation to the zero–one loss—long considered the gold standard for classification performance, yet incompatible with gradient-based optimization due to its non-differentiability. Our method constructs a smooth, order-preserving projection onto the $(n,k)$-dimensional hypersimplex through a constrained optimization framework, leading to a new operator we term Soft-Binary-Argmax. After deriving its mathematical properties, we show how its Jacobian can be efficiently computed and integrated into binary and multiclass learning systems. Empirically, our approach achieves significant improvements in generalization under large-batch training by imposing geometric consistency constraints on the output logits, thereby narrowing the performance gap traditionally observed in large-batch training. Our code is available here 
\url{https://github.com/camilog04/Differentiable-Zero-One-Loss-via-Hypersimplex-Projections}.

\keywords{Differentiable Optimization in Deep Learning \and Differentiable programming \and Large-batch generalization.}
\end{abstract}
\section{Introduction}
Recent developments in machine learning have demonstrated that optimization procedures can be used as fundamental components within end-to-end differentiable systems \cite{DCL, fast_soft_ranking, Differentiating_Parameterized, LTR}. 
Rather than relying solely on traditional neural network layers, these approaches incorporate more structured, often nontrivial computations—such as constrained optimization—directly into the learning pipeline. These components usually have structural or computational properties that have been proven useful in downstream tasks. For instance, Sparsemax, a differentiable projection onto the simplex, produces sparse posterior distributions that are effective as attention mechanisms \cite{sparsemax}. Similarly, Csoftmax, a projection onto the budget polytope, has demonstrated utility in sequence tagging \cite{martins-kreutzer-2017-learning}. This reflects a growing shift toward viewing learning systems as differentiable computational frameworks that blend elements of traditional statistical modeling with algorithmic computation \cite{blondel2024elementsdifferentiableprogramming}. 
In this paper, we focus on crafting a differentiable, order-preserving projection into the $n,k$-dimensional hypersimplex\cite{deloera1995hypersimplex}--a well-studied combinatorics polytope--in composition with a squared loss to generate a close approximation to the zero-one, misclassification loss compatible with modern large-scale differentiable systems.

From a theoretical perspective in machine learning, the goal is to minimize the expected value of a task-specific loss function over a data distribution. For classification, the most natural choice is the zero-one loss, which directly measures misclassification error. However, zero-one loss is non-differentiable and discontinuous, as it depends on a hard threshold decision—yielding gradients of zero almost everywhere and rendering it incompatible with gradient-based optimization. To enable tractable training, modern approaches rely on surrogate losses (e.g., cross-entropy, hinge loss) that are smooth and differentiable \cite{bartlett2006convexity}. These surrogates serve as proxies that approximate the zero-one loss while facilitating efficient optimization. Despite their practicality, such surrogates often exhibit a mismatch with the true evaluation metric, especially under large-batch regimes. This degradation in performance with large batch sizes is the so-called generalization gap \cite{oyedotun2022newperspectiveunderstandinggeneralization}, leading to growing interest in tighter, more faithful approximations to the zero-one loss, under the hypothesis that closer surrogates yield better generalization.


This work introduces a fully differentiable approximation to the zero–one
loss, featuring an efficient forward pass with complexity $\mathcal{O}(n \log n)$ and a
backward pass with $\mathcal{O}(n)$ complexity. This is achieved through our novel
differentiable projection layer, \textbf{Soft-Binary-Argmax@k}. Rather than treating
the output scores as independent, our layer explicitly enforces that the largest $k$
logits correspond to the predicted positive classes. This design ensures that small
perturbations in the input produce coherent, structurally consistent adjustments in
the output, making the Jacobian of the transformation inherently \emph{positionally
aware} with respect to the most confident predictions. The benefits of this approach
are twofold: it allows binary classifiers to express multiple positive outcomes within
a single forward pass, and it extends naturally to the multiclass setting by applying
the same projection principle across one-hot encoded class dimensions. Importantly,
the geometric constraints imposed on the output logits act as a form of regularization
that mitigates the generalization degradation typically observed under large-batch
training, enabling stable optimization and improved predictive performance.

More precisely, our contributions and novelty can be summarized as follows:

\begin{enumerate}
    \item We introduce a differentiable projection layer—a smooth thresholding operator realized via projection onto the interior of the $n,k$-dimensional hypersimplex—termed \textbf{Soft-Binary-Argmax@k}. It provides a differentiable relaxation of the binary \textit{argmax}, reduces to isotonic regression, and enables efficient forward and backward computation on both CPU and GPU.

    \item We propose a smooth, almost-everywhere differentiable loss function for binary classification, the \textbf{HyperSimplex Loss}, and derive its mathematical properties. The loss couples the mean squared error with our projection layer and extends naturally to multiclass.

    \item Through rigorous experimentation, we provide empirical evidence that the proposed loss mitigates the generalization gap and improves performance across multiple classification benchmark datasets.
\end{enumerate}

\section{Related work}
\subsection{Differentiable optimization-based ML}
Optimization-based modeling integrates structure and constraints into machine learning architectures by embedding parameterized \textit{argmin}/\textit{argmax} operations as differentiable layers. Such layers are often formulated as convex, constrained programs, with differentiability achieved via the implicit function theorem applied to the KKT conditions~\cite{Differentiating_Parameterized}. Parallel work has explored differentiable optimization for order-constrained or monotonic outputs, including differentiable isotonic regression operators for smooth, order-aware learning~\cite{ fast_soft_ranking}. To the best of our knowledge, no prior work has introduced a differentiable Euclidean projection onto the $(n,k)$-dimensional hypersimplex—computed via the Pool Adjacent Violators (PAV) algorithm—as a learnable layer. Our approach fills this gap, providing an efficient and theoretically grounded formulation for integrating hypersimplex projections into modern differentiable systems.

\subsection{Generalization gap}
A well-known challenge in modern deep learning is the \emph{generalization gap}, 
where models trained with large batch sizes achieve low training loss but 
exhibit degraded test performance. This phenomenon has been widely observed in neural networks, 
as large batches tend to converge to sharp minima that generalize poorly compared 
to the flatter solutions found by small-batch training~\cite{keskar2016large}. 
Subsequent work has explored remedies such as adaptive learning rate schedules and warmup 
strategies, noise injection and regularization~\cite{hoffer2017train}, 
and stochastic weight averaging~\cite{izmailov2018averaging} to mitigate this effect. However, to the best of our knowledge, our work is the first to address the generalization gap 
\emph{through loss function design}, introducing a principled framework that directly 
links the geometry of the loss landscape to generalization behavior.
\section{Preliminaries}

In supervised multiclass classification, we are given a dataset 
\(\mathcal{D} = \{(x_i, y_i)\}_{i=1}^n\), where each input 
\(x_i \in \mathcal{X} \subset \mathbb{R}^d\) is associated with a categorical label 
\(y_i \in \{1, \dots, C\}\) among \(C\) possible classes.  
Let \(f: \mathcal{X} \to \mathbb{R}^C\) denote a prediction function producing a score vector 
\(f(x_i) = (f_1(x_i), \dots, f_C(x_i))^\top\), where each component \(f_c(x_i)\) reflects the model’s confidence for class \(c\).

The learning objective is to minimize the \emph{multiclass zero–one loss}, which measures the fraction of misclassified samples:
\begin{equation}
\mathcal{L}_{0/1}(f)
= \frac{1}{n} \sum_{i=1}^n 
\mathbb{I}\!\left[\, \hat{y}_i \neq y_i \,\right],
\qquad
\hat{y}_i = \arg\max_{c \in \{1,\dots,C\}} f_c(x_i).
\end{equation}
While $\mathcal{L}_{0/1}$ directly quantifies classification accuracy, it is discontinuous and non-differentiable, making it unsuitable for gradient-based optimization.

To obtain a differentiable surrogate, convex losses are commonly employed.  
Common loss functions in machine learning—such as squared loss, hinge loss, and logistic loss—are convex approximations of the true 0–1 misclassification loss~\cite{bartlett2006convexity}.  
Among these, the squared loss provides the closest approximation to the 0–1 loss on the interval $(0,1)$, making it a natural foundation for our formulation.  
The \emph{multiclass squared loss} penalizes deviations between predicted scores and the corresponding one-hot target encodings, summing over all classes:
\begin{equation}
\mathcal{L}_{\text{sq}}(f)
= \frac{1}{n} \sum_{i=1}^n \sum_{c=1}^C 
\big(f_c(x_i) - \mathbb{I}[y_i = c]\big)^2.
\end{equation}
Equivalently, in matrix form,
\(
\mathcal{L}_{\text{sq}}(f)
= \frac{1}{n} \| F(X) - Y \|_F^2,
\)
where \(F(X) \in \mathbb{R}^{n \times C}\) collects the model outputs and \(Y\) is the one-hot label matrix.  
While this smooth, convex loss provides analytic gradients and serves as a tractable approximation to $\mathcal{L}_{0/1}$, it also suffers from a major drawback: it imposes a quadratic penalty on extreme predicted values, leading to sensitivity to outliers~\cite{huber1964robust,strict_scoring}.  
This limitation motivates our projection-based formulation introduced next, which preserves smoothness while constraining outputs within a geometrically consistent region.

\section{Methodology}
\subsection{Overview}
This work begins by formulating the problem in the binary classification setting, where the objective is to distinguish between positive and negative outcomes. The same geometric principles, however, extend naturally to the multiclass setting, as shown in later sections. Common loss functions in machine learning—such as squared loss, hinge loss, and logistic loss—are convex surrogates of the true 0–1 misclassification loss~\cite{bartlett2006convexity}. However, the true 0–1 loss minimizer lies at one of the vertices of the $(n,k)$-dimensional hypersimplex, as it satisfies two key properties:
\begin{enumerate}
    \item Its entries are binary, i.e., each component of $f(\mathbf{X})$ takes a value in $\{0,1\}$.
    \item For any given sample $\mathbf{y}$ drawn from the distribution of $Y$, a perfect prediction vector should contain exactly $k$ positive entries, matching the number of positives in $\mathbf{y}$; that is, $\|f(\mathbf{X})\|_1 = \|\mathbf{y}\|_1 = k$.
\end{enumerate}

Motivated by the geometry of the optimal solution, we now introduce a series of relaxations that make the learning problem tractable. First, we relax the binary constraint and allow $f(\mathbf{X})$ to take real values in $\mathbb{R}^n$, while encouraging sparsity—pushing predictions as close as possible to $0$ or $1$. To balance smoothness with structural fidelity, we compose our differentiable projection operator, the \textbf{soft-binary-argmax@k}, which produces sparse and nearly binary outputs, with the squared loss, which ensures smooth optimization and stability. This composition yields a surrogate objective that remains differentiable while closely aligning with the discrete geometry of the hypersimplex.

The remainder of this section is organized as follows. We begin by establishing the connection between binary-argmax@k and thresholding. Next, we formulate binary-argmax@k as a projection onto the ${n,k}$ hypersimplex, from which we derive its continuous relaxation, soft-binary-argmax@k, and analyze its key properties. Finally, we combine the soft-binary-argmax@k with a squared loss to define the HyperSimplex loss, and demonstrate its effectiveness in generalization for large batch sizes.

\subsection{Thresholding and the Binary-Argmax@k}
This section establishes an intuitive connection between a real-valued vector 
$\mathbf{x} \in \mathbb{R}^n$ and its binary counterpart in $\{0,1\}^n$. A common discretization method is 
\textit{thresholding}, where entries exceeding a fixed boundary (typically $0.5$) are set to $1$, and the rest to $0$. 
Although simple, this approach ignores relative ordering and offers no control over the number of positive components.

A more structured alternative is the binary-argmax@k operator, denoted $r_k$, which assigns $1$ to the $k$ largest 
entries of $\mathbf{x}$ and $0$ to the remaining $n - k$. Here, the threshold is adaptively defined by the $k$-th largest value 
of $\mathbf{x}$, reducing to standard thresholding when $k = \lceil n/2 \rceil$, where the threshold equals the empirical median 
of the logits.

Formally, let $\mathbf{x} \in \mathbb{R}^n$ be a vector of scores and $k \in \{1, \dots, n\}$.
We define the binary-argmax@k operator as
\begin{equation}
r_k(\mathbf{x}) = \mathbb{I}\!\left(x_i \geq T_k(\mathbf{x})\right),
\quad 
T_k(\mathbf{x}) = \text{$k$-th largest value of } \mathbf{x}.
\end{equation}

This rule ensures exactly $k$ components of $\mathbf{x}$ are set to $1$, enforcing the constraint
\begin{equation}
\|r_k(\mathbf{x})\|_1 = k, 
\quad 
r_k(\mathbf{x}) \in \{0,1\}^n.
\end{equation}

The binary-argmax@k mapping does not provide useful derivatives, hindering gradient-based optimization. 
To address this, we formulate it as a linear optimization problem over the $(n,k)$-dimensional hypersimplex $\Delta^n_k$ 
and introduce a Euclidean regularization term with a temperature parameter, yielding a smooth relaxation—the 
soft-binary-argmax@k. This differentiable formulation preserves the hypersimplex geometry while providing 
informative gradients for end-to-end learning.

\subsection{Binary-Argmax@k: Euclidean Projections onto the Hypersimplex}

The Euclidean projection onto the $(n,k)$-dimensional hypersimplex can be expressed as the solution of a simple regularized linear program:
\begin{equation}
\label{eq:projection_lagrangian_form}
\underset{\mathbf{y} \in \mathbb{R}^n}{\mathrm{argmax}} 
\;\; \langle \mathbf{x}, \mathbf{y} \rangle - \|\mathbf{y}\|_2^2
\quad 
\text{s.t.} \quad 
\mathbf{1}^\top \mathbf{y} = k, \quad 0 \leq \mathbf{y} \leq 1.
\end{equation}
The first term encourages alignment with the input vector $\mathbf{x}$, 
while the quadratic regularization term enforces proximity to the origin, thereby inducing a 
balance between sparsity and fidelity. The affine constraint 
$\mathbf{1}^\top \mathbf{y} = k$ fixes the $\ell_1$ mass of $\mathbf{y}$, 
ensuring exactly $k$ active components, while the box constraint 
$0 \leq \mathbf{y} \leq 1$ confines the solution to the hypercube $[0,1]^n$.

The feasible region defined by these two constraints is precisely the 
$(n,k)$-dimensional hypersimplex:
\begin{equation}
\Delta^n_k = \left\{ \mathbf{y} \in [0,1]^n \; \big| \; 
\sum_{i=1}^n y_i = k \right\}.
\end{equation}
Hence, the optimization problem in~\eqref{eq:projection_lagrangian_form} 
is equivalent to the Euclidean projection over the hypersimplex:
\begin{equation}
\Pi_{\Delta^n_k}(\mathbf{x}) =
\underset{\mathbf{y} \in \Delta^n_k}{\mathrm{argmin}} 
\; \|\mathbf{x} - \mathbf{y}\|_2^2.
\end{equation}
Since $\Delta^n_k$ is convex and compact, this problem admits a unique solution. 
Geometrically, $\Pi_{\Delta^n_k}(\mathbf{x})$ corresponds to the point 
within $\Delta^n_k$ that lies closest to $\mathbf{x}$ in Euclidean distance, 
and algebraically, it coincides with the binary vector that activates 
the $k$ largest components of $\mathbf{x}$, the binary-argmax@k.

\begin{figure} 
    \centering
    \vspace{-6pt} 
    \includegraphics[width=\textwidth, height=0.35\textwidth, keepaspectratio]{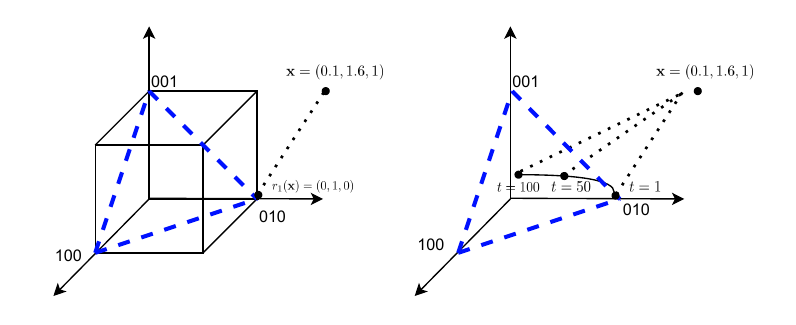}
    \caption{Binary-argmax@k of a point $\mathbf{x} = (0.1, 1.6, 1)$ into the exterior of the Hypersimplex (left). At $k=1$, the solution is $(0,1,0)$. Introducing temperature to the program yields an interior solution (right), i.e., the soft-binary-argmax@1. In $\mathbb{R}^3$ different $k$ values yield points on a standard simplex, but in higher dimensions yields a point on the hypersimplex. In $\mathbb{R}^4$ with $k=2$, the solution lies on an octahedron~\cite{amos2019limitedmultilabelprojectionlayer}.}
    \vspace{-8pt} 
\end{figure}

\subsection{Soft-Binary-Argmax@k: A Differentiable Approximation}

The projection $\Pi_{\Delta^n_k}(\mathbf{x})$ provides a geometric mapping from a continuous vector $\mathbf{x}$ to its structured binary counterpart, but it remains piecewise constant and thus non-differentiable. Small perturbations in $\mathbf{x}$ can abruptly change the identity of the top-$k$ elements, yielding discontinuities and zero gradients almost everywhere. Consequently, the hard binary-argmax@k operator is incompatible with gradient-based optimization.

\paragraph{Temperature-Scaled Relaxation.}
To obtain a smooth approximation, we introduce a temperature parameter $\tau>0$ that scales the regularization strength in the projection objective:
\begin{equation}
\label{eq:soft_projection_def}
\underset{\mathbf{y} \in \Delta^n_k}{\mathrm{argmin}} 
\left( \tau \|\mathbf{y}\|_2^2 - 2\langle \mathbf{x}, \mathbf{y} \rangle \right) =\underset{\mathbf{y} \in \Delta^n_k}{\mathrm{argmin}} 
\left( \|\mathbf{y}\|_2^2 - 2\Big\langle \tfrac{\mathbf{x}}{\tau}, \mathbf{y} \Big\rangle \right)
\end{equation}

leading to the compact expression
\begin{equation}
\label{eq:soft_projection_scaled}
\Pi_{\tau}\!\left( {\mathbf{x}} \right)
= 
\underset{\mathbf{y} \in \Delta^n_k}{\mathrm{argmin}}
\Big\| \mathbf{y} - \tfrac{\mathbf{x}}{\tau} \Big\|_2^2 = \Pi_{\Delta^n_k}\!\left( \tfrac{\mathbf{x}}{\tau} \right)
\end{equation}
As $\tau \!\to\! 0$, the operator recovers the discontinuous hard projection, while larger $\tau$ values 
yield smoother outputs closer to the hypersimplex—defining the soft-Binary-Argmax@k operator.

\begin{proposition}[Differentiability a.e]\label{prop:ae-diff}
Fix $k\in\{1,\dots,n\}$ and $\tau>0$. The mapping
\[
F_\tau:\ \mathbb{R}^n \to \Delta^n_k,\qquad 
F_\tau(\mathbf{x}) :=\Pi_{\Delta^n_k}\!\left(\tfrac{\mathbf{x}}{\tau}\right)
\]
is $(1/\tau)$-Lipschitz, hence differentiable almost everywhere (a.e.) in $\mathbb{R}^n$.
\end{proposition}

\begin{proof}
The Euclidean projection onto a closed convex set in a Hilbert space is nonexpansive:
$\|\Pi_C(\mathbf{u})-\Pi_C(\mathbf{v})\|_2 \le \|\mathbf{u}-\mathbf{v}\|_2$ for all $\mathbf{u},\mathbf{v}$.
With $C=\Delta^n_k$ and $\mathbf{u}=\mathbf{x}/\tau$, $\mathbf{v}=\mathbf{z}/\tau$,
\[
\bigl\|F_\tau(\mathbf{x})-F_\tau(\mathbf{z})\bigr\|
= \Bigl\|\Pi_{\Delta^n_k}\!\left(\tfrac{\mathbf{x}}{\tau}\right) - \Pi_{\Delta^n_k}\!\left(\tfrac{\mathbf{z}}{\tau}\right)\Bigr\|
\le \Bigl\|\tfrac{\mathbf{x}}{\tau}-\tfrac{\mathbf{z}}{\tau}\Bigr\|
= \tfrac{1}{\tau}\,\|\mathbf{x}-\mathbf{z}\|.
\]
Thus $F_\tau$ is $(1/\tau)$-Lipschitz. By Rademacher’s theorem, every Lipschitz map on $\mathbb{R}^n$ is differentiable a.e., proving the claim.
\end{proof}

\begin{proposition}[Order preservation]\label{prop:order_preservation}
The projection solution \( y_i = \Pi_{\Delta^n_k}\!\left(\tfrac{x_i}{\tau}\right) \) is order preserving; that is, if \( x_1/\tau \ge x_2/\tau \ge \cdots \ge x_n/\tau \), then the projected coordinates satisfy \( y_1 \ge y_2 \ge \cdots \ge y_n \).
\end{proposition}

\begin{proof}
From the KKT conditions of the Lagrangian associated with \eqref{eq:soft_projection_scaled}, stationarity and complementarity yield, for each \(i\),
\[
y_i = \mathrm{clip}\!\left(x_i/\tau - \tfrac{\lambda}{2},\, 0,\, 1\right),
\]
where the multiplier \( \lambda \) is uniquely determined to satisfy the equality constraint \( \sum_i y_i = k \). Since the mapping \(t \mapsto \mathrm{clip}\!\left(t-\tfrac{\lambda}{2},0,1\right)\) is monotone nondecreasing in \(t\), it follows that \(x_1/\tau \ge \cdots \ge x_n/\tau \Rightarrow y_1 \ge \cdots \ge y_n\). See  \cite{gomez2025projection} for details.

\end{proof}

\begin{corollary}[Computation]
Since the projection is order preserving (Proposition~\ref{prop:order_preservation}), adding a monotonicity constraint does not change the solution. Hence, for any sorted input $\mathbf{x}/\tau$, the projection can be computed via a reduction to isotonic regression:
\[
\Pi(\mathbf{x}/\tau)
= \underset{
\substack{
\mathbf{y} \in [0,1]^n,\\
\mathbf{1}^\top \mathbf{y} = k,\\
y_1 \ge \cdots \ge y_n
}
}{\arg\min}\;
\left\| \tfrac{\mathbf{x}}{\tau} - \mathbf{y} \right\|^2.
\]
The feasible set is closed and convex, ensuring a unique and differentiable solution. This reduces to a standard isotonic projection problem, solvable efficiently via the pool-adjacent-violators (PAV) algorithm \cite{Best1990} in $\mathcal{O}(n \log n)$ time.
\end{corollary}

\subsection{The HyperSimplex loss}
We now compose our projection operator with the squared loss to define a smooth surrogate for binary classification. The squared loss provides high fidelity to the zero–one objective within $(0,1)$ but can be dominated by large-magnitude predictions. By composing it with the projection operator $\Pi_{\Delta^k_n}$, we constrain predictions to the hypersimplex, preventing any coordinate from overtaking the loss while preserving the discrete geometry of the solution.

Formally, for $\mathbf{x}, \mathbf{y} \in \mathbb{R}^n$, define
\[
\hat{\mathbf{y}} = \Pi_{\Delta^k_n}\!\left(\frac{\mathbf{x}}{\tau}\right),
\qquad 
L(\mathbf{x}, \mathbf{y}) = \tfrac{1}{2}\|\hat{\mathbf{y}} - \mathbf{y}\|_2^2,
\]
where $\tau > 0$ controls the smoothness of the relaxation.  
The gradient with respect to $\mathbf{x}$ follows from the chain rule:
\[
\nabla_{\mathbf{x}} L(\mathbf{x}, \mathbf{y}) 
= \frac{1}{\tau}\, J_{\Pi}\!\left(\tfrac{\mathbf{x}}{\tau}\right)(\hat{\mathbf{y}} - \mathbf{y}),
\]
where $J_{\Pi}$ denotes the Jacobian of the projection operator $\Pi_{\Delta^k_n}$.  
Let $A = \{i : 0 < \hat{y}_i < 1\}$ denote the active coordinates. On this set, the Jacobian acts as
\[
J_{\Pi} = I_{|A|} - \tfrac{1}{|A|}\mathbf{1}\mathbf{1}^\top,
\]
yielding the component-wise gradient
\[
(\nabla_{\mathbf{x}} L)_i =
\begin{cases}
\dfrac{1}{\tau}\!\left[(\hat{y}_i - y_i) - \dfrac{1}{|A|}\!\sum_{j \in A}(\hat{y}_j - y_j)\right], & i \in A, \\[6pt]
0, & i \notin A.
\end{cases}
\]
At boundary points where some $\hat{y}_i \in \{0,1\}$, the mapping is only directionally differentiable, and any subgradient consistent with this Jacobian form is valid.

\subsection{Extension to Multiclass Classification}

The formulation extends naturally to the multiclass setting.  
For each class $c \in \{1,\dots,C\}$ with logits $\mathbf{x}^{(c)} \in \mathbb{R}^n$, one-hot target $\mathbf{y}^{(c)}$, and temperature $\tau_c > 0$, we project onto the $(n,k_c)$-hypersimplex:
\[
\mathbf{p}^{(c)} = \Pi_{\Delta^n_{k_c}}\!\left(\tfrac{\mathbf{x}^{(c)}}{\tau_c}\right).
\]
The total loss is
\[
\mathcal{L}(X,Y) = \tfrac{1}{2}\sum_{c=1}^C
\bigl\|
\mathbf{p}^{(c)} - \mathbf{y}^{(c)}
\bigr\|_2^2,
\qquad
\nabla_{\mathbf{x}^{(c)}} \mathcal{L}
=
\tfrac{1}{\tau_c}
J_{\Pi}\!\left(\tfrac{\mathbf{x}^{(c)}}{\tau_c}\right)
(\mathbf{p}^{(c)} - \mathbf{y}^{(c)}).
\]
This provides a smooth, per-class projection framework that preserves hypersimplex structure while remaining fully differentiable. During the learning process each $k_c$ is set to match the expected number of positive responses for class $c$. 
\section{Experiments}
For our experiments, we evaluate the effectiveness of the proposed HyperSimplex loss in reducing the generalization gap compared to standard classification losses, including Cross-Entropy, Hinge, and Mean Squared Error (MSE, without projection). This setup also serves as an ablation study to isolate the contribution of our projection layer, verifying that incorporating geometric constraints on the output logits yields more consistent performance across batch sizes than the MSE objective alone.

\subsection{Datasets}
We conduct experiments on two standard image classification benchmarks: CIFAR-10~\cite{krizhevsky2009learning} and Fashion-MNIST~\cite{xiao2017fashion}. 
CIFAR-10 consists of 60{,}000 color images of size $32 \times 32$ pixels, split into 50{,}000 training and 10{,}000 test samples across 10 object categories. 
Fashion-MNIST contains 70{,}000 grayscale images of size $28 \times 28$ pixels, divided into 60{,}000 training and 10{,}000 test images from 10 clothing categories, serving as a more challenging replacement for the original MNIST dataset.

\subsection{Experimental Setup}

We employed a standard convolutional neural network (CNN) for multiclass image classification, consisting of four convolutional layers, each followed by batch normalization, max pooling, and ReLU activation. The final feature map is flattened and passed through two fully connected layers, with the last layer producing class logits. The datasets were preprocessed using random cropping, horizontal flipping, and per-channel normalization, and randomly split into training and test sets. All experiments were implemented in PyTorch~\cite{paszke2019pytorch} and executed on 32-core AMD Ryzen Threadripper PRO 5975WX CPU with 503\,GB of RAM and three NVIDIA RTX 6000 Ada Generation GPUs, each with 48\,GB of VRAM.

To ensure statistical robustness, each configuration was trained using five independent random seeds, varying both model initialization and data splits. We evaluated four loss functions—our proposed \textbf{HyperSimplex loss} and three widely used baselines: Cross-Entropy, Hinge, and Mean Squared Error (MSE)--across seven batch sizes (128, 256, 512, 1024, 2048, 4096 and 8192) on both CIFAR-10 and Fashion-MNIST. In total, this resulted in 280 training runs. For each configuration, we recorded the maximum test accuracy achieved per loss function and batch size, and assessed differences against the Cross-Entropy baseline using paired $t$-tests at the 10\% significance level. This experimental design provides a rigorous and statistically grounded comparison, isolating the contribution of the HyperSimplex formulation to generalization stability under varying batch regimes.

\subsection{Results}

\includegraphics[width=0.49\linewidth]{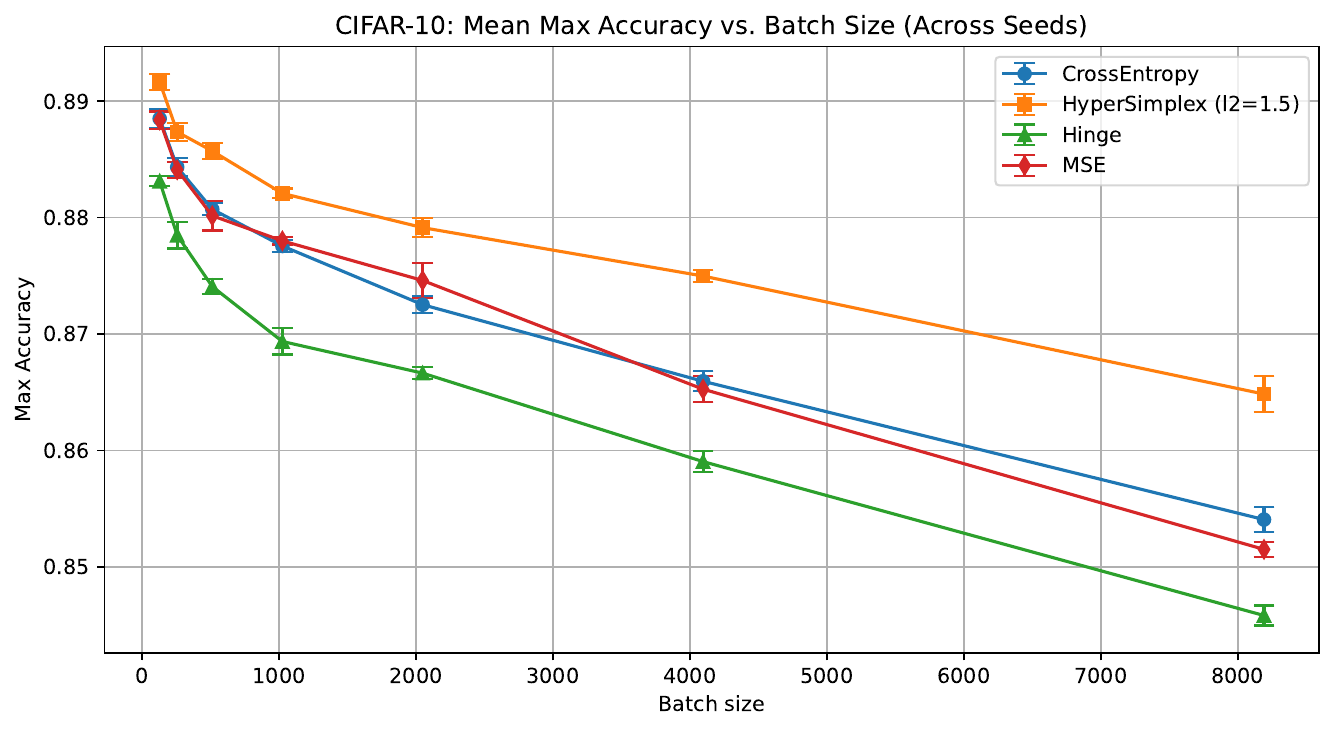}
\includegraphics[width=0.49\linewidth]{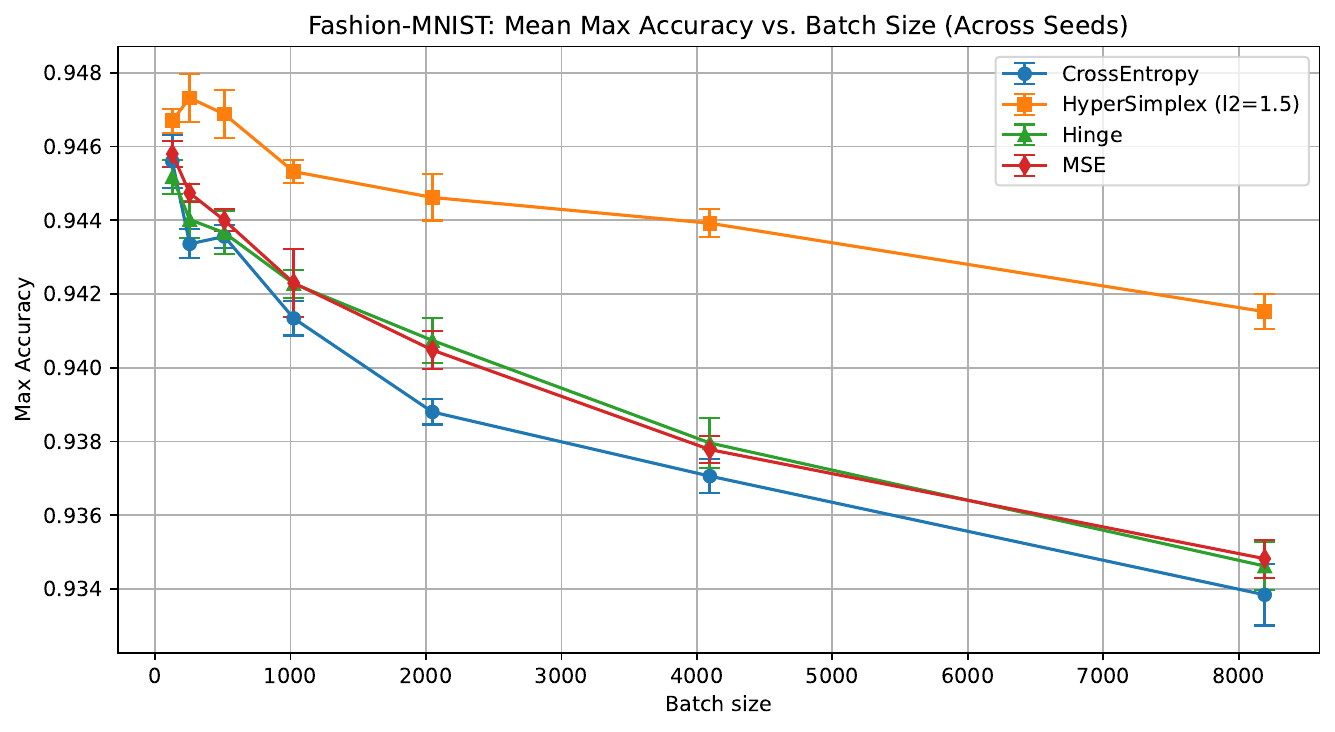}

For CIFAR-10, all seven configurations report positive mean accuracy differences, and all ($100\%$) show statistically significant improvements at the 10\% level ($p<0.1$).  
For Fashion-MNIST, six of seven configurations ($\approx86\%$) also achieve significance, with only the smallest batch size ($128$) falling above the 10\% threshold.  
Across both datasets, therefore, 13 of 14 total comparisons ($\approx93\%$) demonstrate statistically significant gains, indicating that the proposed loss systematically outperforms cross-entropy across a wide range of training conditions.

These findings confirm that the HyperSimplex loss maintains accuracy stability at smaller batch sizes while mitigating the degradation observed in cross-entropy as batch size increases. This supports its effectiveness as a smooth, geometry-consistent surrogate that enhances generalization and robustness in large-batch training regimes.

\begin{table*}[htbp]
\centering
\caption{Batch-wise Comparison of Cross-Entropy (CE) vs. HyperSimplex (HS) Losses on CIFAR-10 and Fashion-MNIST. The \textbf{highlighted} scores are statistically significant at a 10\% level of significance.} 
\label{tab:batchwise_ce_vs_hs}
\resizebox{\textwidth}{!}{%
\begin{tabular}{c c c c c c  c c c c c}
\toprule
\multicolumn{6}{c}{\textbf{CIFAR-10}} & \multicolumn{5}{c}{\textbf{FashionMNIST}} \\
\cmidrule(lr){1-6}\cmidrule(lr){7-11}
\textbf{Batch} & \textbf{CE} & \textbf{HS (ours)} & \textbf{$\Delta$} & \textbf{t-stat} & \textbf{p-val}
& \textbf{CE} & \textbf{HS (ours)} & \textbf{$\Delta$} & \textbf{t-stat} & \textbf{p-val} \\
\midrule
128  & 0.8885 & \textbf{0.8917} & 0.0032 &  2.29 & \textbf{0.084} 
     & 0.9456 & 0.9467 & 0.0011 &  1.31 & 0.262 \\
256  & 0.8843 & \textbf{0.8874} & 0.0030 &  3.47 & \textbf{0.026} 
     & 0.9434 & \textbf{0.9473} & 0.0040 &  4.34 & \textbf{0.012} \\
512  & 0.8807 & \textbf{0.8857} & 0.0050 &  6.08 & \textbf{<0.01} 
     & 0.9436 & \textbf{0.9469} & 0.0033 &  5.51 & \textbf{<0.01} \\
1024 & 0.8776 & \textbf{0.8821} & 0.0045 &  9.78 & \textbf{<0.01} 
     & 0.9413 & \textbf{0.9453} & 0.0040 &  6.67 & \textbf{<0.01} \\
2048 & 0.8725 & \textbf{0.8791} & 0.0066 &  7.81 & \textbf{<0.01} 
     & 0.9388 & \textbf{0.9446} & 0.0058 & 12.97 & \textbf{<0.01} \\
4096 & 0.8659 & \textbf{0.8750} & 0.0090 & 13.14 & \textbf{<0.01} 
     & 0.9371 & \textbf{0.9439} & 0.0069 &  9.68 & \textbf{<0.01} \\
8192 & 0.8541 & \textbf{0.8648} & 0.0108 &  8.52 & \textbf{<0.01} 
     & 0.9338 & \textbf{0.9415} & 0.0077 & 14.99 & \textbf{<0.01} \\
\bottomrule
\end{tabular}
}
\end{table*}

\subsection{Other Experiments: Cross-domain Validation }
Additional GBRT results on tabular datasets for classification are reported in Appendix~\ref{app:A},
showing that the HyperSimplex loss also improves out-of-sample generalization beyond the neural settings studied
in the main text.
\section{Conclusion}
We introduced the soft-binary-argmax@k, a differentiable projection onto the
interior of the $(n,k)$-dimensional hypersimplex, and established its key
properties—differentiability, order preservation, and efficient GPU computation.
We showed how this operator integrates naturally into end-to-end learning
systems, and used it to construct a surrogate to the zero--one loss for binary and
multiclass settings, with statistically significant reductions in the generalization gap.

Owing to its close alignment with the true zero--one objective, the proposed
HyperSimplex loss improves generalization under large-batch training.
Additional cross-domain evaluations on tabular data further suggest that the benefits of the
projection extend beyond neural models. Future work will explore applications to
contrastive learning objectives and structured prediction.
\appendix
\section{Appendix: Cross-Domain Tabular Results}
\label{app:A}
\begin{table}[h!]\centering\scriptsize
\begin{tabular}{lcccccccc}
\toprule
Dataset & Higgs & Flight & KDD10 & KDD12 & Criteo & Avazu & KKBox & MovieLens \\
\midrule
Cross Entropy      & 0.823 & 0.773 & 0.826 & 0.724 & 0.774 & 0.738 & 0.777 & 0.827 \\
HyperSimplex      & \textbf{0.846} & \textbf{0.778} & \textbf{0.849} & \textbf{0.729} &
          \textbf{0.796} & \textbf{0.741} & \textbf{0.797} & \textbf{0.828} \\
\bottomrule
\end{tabular}
\end{table}

%
%

\begin{credits}

\end{credits}
%
%
%
%

\bibliographystyle{splncs04}
\bibliography{references}

@article{DCL,
  author       = {Akshay Agrawal and
                  Brandon Amos and
                  Shane T. Barratt and
                  Stephen P. Boyd and
                  Steven Diamond and
                  J. Zico Kolter},
  title        = {Differentiable Convex Optimization Layers},
  journal      = {CoRR},
  volume       = {abs/1910.12430},
  year         = {2019},
  url          = {http://arxiv.org/abs/1910.12430},
  eprinttype    = {arXiv},
  eprint       = {1910.12430},
  bibsource    = {dblp computer science bibliography, https://dblp.org}
}

@article{sparsemax,
  author       = {Andr{\'{e}} F. T. Martins and
                  Ram{\'{o}}n Fernandez Astudillo},
  title        = {From Softmax to Sparsemax: {A} Sparse Model of Attention and Multi-Label
                  Classification},
  journal      = {CoRR},
  volume       = {abs/1602.02068},
  year         = {2016},
  url          = {http://arxiv.org/abs/1602.02068},
  eprinttype    = {arXiv},
  eprint       = {1602.02068},
  bibsource    = {dblp computer science bibliography, https://dblp.org}
}

@inproceedings{martins-kreutzer-2017-learning,
    title = "Learning What`s Easy: Fully Differentiable Neural Easy-First Taggers",
    author = "Martins, Andr{\'e} F. T.  and
      Kreutzer, Julia",
    editor = "Palmer, Martha  and
      Hwa, Rebecca  and
      Riedel, Sebastian",
    booktitle = "Proceedings of the 2017 Conference on Empirical Methods in Natural Language Processing",
    month = sep,
    year = "2017",
    address = "Copenhagen, Denmark",
    publisher = "Association for Computational Linguistics",
    url = "https://aclanthology.org/D17-1036/",
    doi = "10.18653/v1/D17-1036",
    pages = "349--362"
}

@article{Best1990,
  author    = {M. J. Best and N. Chakravarti},
  title     = {Active set algorithms for isotonic regression: a unifying framework},
  journal   = {Mathematical Programming},
  volume    = {47},
  number    = {1-3},
  pages     = {425--439},
  year      = {1990},
  publisher = {Springer},
  doi       = {10.1007/BF01580880}
}

@article{strict_scoring,
  title={Strictly proper scoring rules, prediction, and estimation},
  author={Gneiting, Tilmann and Raftery, Adrian E},
  journal={Journal of the American Statistical Association},
  volume={102},
  number={477},
  pages={359--378},
  year={2007},
  publisher={Taylor \& Francis}
}

@article{huber1964robust,
  title={Robust estimation of a location parameter},
  author={Huber, Peter J},
  journal={The Annals of Mathematical Statistics},
  volume={35},
  number={1},
  pages={73--101},
  year={1964},
  publisher={Institute of Mathematical Statistics}
}

@article{Differentiating_Parameterized,
  author       = {Stephen Gould and
                  Basura Fernando and
                  Anoop Cherian and
                  Peter Anderson and
                  Rodrigo Santa Cruz and
                  Edison Guo},
  title        = {On Differentiating Parameterized Argmin and Argmax Problems with Application
                  to Bi-level Optimization},
  journal      = {CoRR},
  volume       = {abs/1607.05447},
  year         = {2016},
  url          = {http://arxiv.org/abs/1607.05447},
  eprinttype    = {arXiv},
  eprint       = {1607.05447},
  bibsource    = {dblp computer science bibliography, https://dblp.org}
}

@misc{amos2019limitedmultilabelprojectionlayer,
      title={The Limited Multi-Label Projection Layer}, 
      author={Brandon Amos and Vladlen Koltun and J. Zico Kolter},
      year={2019},
      eprint={1906.08707},
      archivePrefix={arXiv},
      primaryClass={cs.LG},
      url={https://arxiv.org/abs/1906.08707}, 
}

@INPROCEEDINGS{LTR,
  author={Gomez, Camilo and Wang, Pengyang and Fu, Yanjie},
  booktitle={2023 IEEE International Conference on Data Mining (ICDM)}, 
  title={Metric-agnostic Learning-to-Rank via Boosting and Rank Approximation}, 
  year={2023},
  volume={},
  number={},
  pages={1043-1048},
  doi={10.1109/ICDM58522.2023.00121}}

@inproceedings{fast_soft_ranking,
  title={Fast differentiable sorting and ranking},
  author={Blondel, Mathieu and Teboul, Olivier and Berthet, Quentin and Djolonga, Josip},
  booktitle={International Conference on Machine Learning},
  pages={950--959},
  year={2020},
  organization={PMLR}
}

@phdthesis{gomez2025projection,
  author = {Gomez, Camilo},
  title  = {Fast Differentiable Projection Layers onto High-Dimensional Polytopes for Large-Scale Predictive Modeling},
  school = {University of Central Florida},
  year   = {2025},
  url    = {https://stars.library.ucf.edu/etd2024/450}
}

@misc{blondel2024elementsdifferentiableprogramming,
      title={The Elements of Differentiable Programming}, 
      author={Mathieu Blondel and Vincent Roulet},
      year={2024},
      eprint={2403.14606},
      archivePrefix={arXiv},
      primaryClass={cs.LG},
      url={https://arxiv.org/abs/2403.14606}, 
}

@misc{oyedotun2022newperspectiveunderstandinggeneralization,
      title={A New Perspective for Understanding Generalization Gap of Deep Neural Networks Trained with Large Batch Sizes}, 
      author={Oyebade K. Oyedotun and Konstantinos Papadopoulos and Djamila Aouada},
      year={2022},
      eprint={2210.12184},
      archivePrefix={arXiv},
      primaryClass={cs.LG},
      url={https://arxiv.org/abs/2210.12184}, 
}

@article{bartlett2006convexity,
  title={Convexity, classification, and risk bounds},
  author={Bartlett, Peter L and Jordan, Michael I and McAuliffe, Jon D},
  journal={Journal of the American Statistical Association},
  volume={101},
  number={473},
  pages={138--156},
  year={2006},
  publisher={Taylor & Francis}
}

@inproceedings{paszke2019pytorch,
  title={PyTorch: An Imperative Style, High-Performance Deep Learning Library},
  author={Paszke, Adam and Gross, Sam and Massa, Francisco and Lerer, Adam and Bradbury, James and Chanan, Gregory and Killeen, Trevor and Lin, Zeming and Gimelshein, Natalia and Antiga, Luca and Desmaison, Alban and K{\"o}pf, Andreas and Yang, Edward and DeVito, Zachary and Raison, Martin and Tejani, Alykhan and Chilamkurthy, Sasank and Steiner, Benoit and Fang, Lu and Bai, Junjie and Chintala, Soumith},
  booktitle={Advances in Neural Information Processing Systems},
  pages={8024--8035},
  year={2019}
}

@inproceedings{keskar2016large,
  author    = {Nitish Shirish Keskar and Dheevatsa Mudigere and Jorge Nocedal and Mikhail Smelyanskiy and Ping Tak Peter Tang},
  title     = {On Large-Batch Training for Deep Learning: Generalization Gap and Sharp Minima},
  booktitle = {Proceedings of the 5th International Conference on Learning Representations (ICLR)},
  year      = {2017},
  url       = {https://openreview.net/forum?id=H1oyRlYgg}
}

@inproceedings{hoffer2017train,
  author    = {Elad Hoffer and Itay Hubara and Daniel Soudry},
  title     = {Train Longer, Generalize Better: Closing the Generalization Gap in Large Batch Training of Neural Networks},
  booktitle = {Proceedings of the 31st Conference on Neural Information Processing Systems (NeurIPS)},
  year      = {2017},
  pages     = {1731--1741}
}

@inproceedings{izmailov2018averaging,
  author    = {Pavel Izmailov and Dmitrii Podoprikhin and Timur Garipov and Dmitry Vetrov and Andrew Gordon Wilson},
  title     = {Averaging Weights Leads to Wider Optima and Better Generalization},
  booktitle = {Proceedings of the 34th Conference on Uncertainty in Artificial Intelligence (UAI)},
  year      = {2018},
  pages     = {876--885},
  url       = {https://arxiv.org/abs/1803.05407}
}

@techreport{krizhevsky2009learning,
  title={Learning Multiple Layers of Features from Tiny Images},
  author={Krizhevsky, Alex},
  year={2009},
  institution={University of Toronto},
  number={TR-2009},
  url={https://www.cs.toronto.edu/~kriz/learning-features-2009-TR.pdf}
}

@article{xiao2017fashion,
  title={Fashion-MNIST: A Novel Image Dataset for Benchmarking Machine Learning Algorithms},
  author={Xiao, Han and Rasul, Kashif and Vollgraf, Roland},
  journal={arXiv preprint arXiv:1708.07747},
  year={2017},
  url={https://arxiv.org/abs/1708.07747}
}

@article{deloera1995hypersimplex,
  title={Gr{\"o}bner bases and triangulations of the second hypersimplex},
  author={De Loera, Jes{\'u}s A. and Sturmfels, Bernd and Thomas, Rekha R.},
  journal={Combinatorica},
  volume={15},
  number={3},
  pages={409--424},
  year={1995},
  publisher={Springer},
  doi={10.1007/BF01299745}
}
\end{document}